\documentclass[journal]{IEEEtran}

\usepackage{amsmath, amssymb, amsfonts, amsthm}
\usepackage{graphicx}
\usepackage{cite}
\usepackage{bbm}
\usepackage{mathtools}
\usepackage{microtype}
\usepackage{hyperref}

\usepackage{tikz}
\usetikzlibrary{arrows.meta,positioning,calc,fit}

 \newtheorem{theorem}{Theorem}
 \newtheorem{lemma}{Lemma}
 \newtheorem{definition}{Definition}
 \newtheorem{remark}{Remark}

\begin{document}


\title{The Query Channel: Information-Theoretic Limits of Masking-Based Explanations}


\author{
Erciyes Karakaya*, Ozgur Ercetin\textsuperscript{\dag}\\
*Department of Electrical and Computer Engineering, University of Maryland, College Park, USA\\
\textsuperscript{\dag}Faculty of Engineering and Natural Sciences, Sabanci University, Turkiye\\
}

\maketitle


\begin{abstract}
Masking-based post-hoc explanation methods, such as KernelSHAP and LIME, estimate local feature importance by querying a black-box model under randomized perturbations. This paper formulates this procedure as communication over a \emph{query channel}, where the latent explanation acts as a message and each masked evaluation is a channel use. Within this framework, the complexity of the explanation is captured by the entropy of the hypothesis class, while the query interface supplies information at a rate determined by an identification capacity per query. We derive a strong converse showing that, if the explanation rate exceeds this capacity, the probability of exact recovery necessarily converges to one in error for any sequence of explainers and decoders. We also prove an achievability result establishing that a sparse maximum-likelihood decoder attains reliable recovery when the rate lies below capacity. A Monte Carlo estimator of mutual information yields a nonasymptotic query benchmark that we use to compare optimal decoding with Lasso- and OLS-based procedures that mirror LIME and KernelSHAP. Experiments reveal a range of query budgets where information theory permits reliable explanations but standard convex surrogates still fail. Finally, we interpret super-pixel resolution as a source-coding choice that sets the entropy of the explanation and show how Gaussian noise and nonlinear curvature degrade the query channel, induce waterfall and error-floor behavior, and render high-resolution explanations unattainable.
\end{abstract}


\section{Introduction}

Masking-based post-hoc explanation methods, such as KernelSHAP and LIME, infer local feature importance by repeatedly querying a model under randomized binary masks. Although widely deployed, these procedures often show large variability with respect to the number of queries, the masking distribution, or the resolution of the feature representation. Existing explanations for this variability focus on algorithmic instability, regression ill-conditioning, or segmentation heuristics. However, these accounts remain tied to specific implementations and do not address a more fundamental question: \emph{Is reliable explanation possible given the query interface itself?}

This work develops an information-theoretic perspective in which explanation extraction is treated as communication over a noisy channel. The unknown explanation vector plays the role of a message, binary masks constitute channel inputs, and oracle evaluations $f(x^{(t)})$ serve as channel outputs. Within this model, the explainer acts as a decoder attempting to recover a sparse message from noisy linear projections. This framing enables a direct application of Shannon’s theory and, Wolfowitz’s strong converse: when the explanation rate exceeds channel capacity, \emph{no decoder, regardless of computational power, can reliably recover the explanation}.

Under this interpretation, the number of queries $T$ is the blocklength of the communication system, the sparsity $k$ and dimensionality $d$ determine the combinatorial entropy of the message, and the masking distribution defines the per-query mutual information of the channel. This yields a nonasymptotic, operational threshold $T_{\mathrm{IT}}$—the minimum number of queries required for reliable decoding—below which all algorithms must fail and above which an optimal decoder succeeds.

The information-theoretic approach further clarifies a long-standing practical issue in image explanations: finer super-pixel resolutions often degrade explanation quality even when the underlying model behaves smoothly. In our framework, increasing the number of segments increases the entropy of the explanation source, eventually exceeding the information that the query channel can convey under a fixed budget. This identifies a critical resolution beyond which explanations are mathematically impossible rather than algorithmically deficient.

Together, these results show that masking-based explainers are governed by the same rate--capacity tradeoffs as classical communication systems. Reliable explanation is achievable only when the entropy of the explanation source lies below the mutual information delivered by the query channel; when this inequality is violated, instability and error floors are unavoidable, regardless of algorithmic details.

\subsection{Related Works}

Our work bridges the gap between the information-theoretic limits of high-dimensional statistics and the reliability analysis of post-hoc explanation methods. We categorize the relevant literature into three primary domains.

\textbf{Information-Theoretic Limits of Sparse Recovery.}
The fundamental limits of recovering sparse vectors from noisy linear measurements are well-established in the field of Compressed Sensing (CS)~\cite{donoho2006cs, candes2006robust, donoho2009message}. While initial works focused on noiseless or bounded-error recovery, Wainwright~\cite{wainwright2009} rigorously proved that for the exact support recovery of a $k$-sparse vector in $d$ dimensions under Gaussian noise, the number of measurements must scale as $T > \mathcal{O}(k \log d)$. Reeves and Gastpar~\cite{reeves2012} extended this analysis to approximate recovery, deriving sharp phase transitions that separate achievable regions from impossible ones based on the Signal-to-Noise Ratio (SNR). Other works have refined these limits for specific matrix ensembles~\cite{wu2010optimal} and sparse measurement matrices~\cite{wang2010sparse}.

In the broader context of machine learning, information-theoretic tools have been applied to generalization bounds~\cite{xu2017information, achille2018information} and representation learning (e.g., the Information Bottleneck~\cite{tishby2000information}). However, these works typically focus on the sample complexity of \textit{learning} a model from data. Our work specifically targets the \textit{interpretation} of an existing model. We extend the classical CS results to the \textit{Active Query} setting unique to XAI, where the measurement matrix $Z$ is not a fixed Gaussian ensemble but is generated dynamically by the explanation algorithm's sampling policy (e.g., LIME's local perturbation or SHAP's axiomatic kernel).

\textbf{Mechanism and Instability of Post-hoc Explanations.}
Perturbation-based methods such as KernelSHAP~\cite{lundberg2017shap, strumbelj2014explaining} and LIME~\cite{ribeiro2016lime} operate by regressing model outputs on sampled binary masks. While powerful, these methods suffer from well-documented instability~\cite{alvarez2018, slack2020fooling, ghorbani2019fragile}, where small changes in inputs or sampling seeds yield vastly different explanations. Prior theoretical analyses have largely attributed this to optimization variance. For instance, Alvarez-Melis and Jaakkola~\cite{alvarez2018} formalized stability as Lipschitz continuity, identifying LIME's local linearity assumption as a source of violation. Garreau and von Luxburg~\cite{garreau2021explaining} derived asymptotic bounds for LIME's coefficients, focusing on the variance of the estimator under finite sampling. Other works have proposed evaluating fidelity via sanity checks~\cite{adebayo2018sanity}, sensitivity metrics~\cite{yeh2019infidelity}, or retraining benchmarks~\cite{hooker2019benchmark}, and improving estimation via unbiased estimators~\cite{covert2021improving}.

Our work offers a complementary perspective to these optimization-centric views: we reframe instability not as a failure of convergence, but as a \textit{channel outage} event. By viewing the sampling process as a noisy communication channel, we argue that instability arises fundamentally when the query budget $T$ is insufficient to convey the entropy of the explanation $H(S)$, regardless of the estimator's variance.

\textbf{Strong Converses in Channel Coding.}
Classical results by Shannon~\cite{shannon1948mathematical} and Wolfowitz~\cite{wolfowitz1957strong} establish that for communication rates exceeding channel capacity, the probability of decoding error converges to one exponentially fast. This phenomenon, known as the Strong Converse~\cite{cover2006elements}, provides a hard thermodynamic limit on recoverability. The present work demonstrates that masking-based explainers obey this same law. By adapting Wolfowitz’s theorem to the masked-query channel, we provide an operational meaning to the ``impossibility'' of explanation: if the entropy of the desired explanation resolution exceeds the information capacity of the query interface, reliable feature attribution is mathematically impossible.

\textbf{Recent work by Rao~\cite{rao2025limitsaiexplainabilityalgorithmic} }derives
fundamental limits on explainability using Algorithmic Information Theory,
showing that any explanatory model whose Kolmogorov complexity is
substantially smaller than that of the target predictor must incur non-negligible
approximation error (the \emph{complexity gap}). Their results characterize when
a valid explanation \emph{can exist} given the structural complexity of $f$.  

Our work addresses a different axis of limitation based on Shannon Information
Theory: even when a valid, low-complexity explanation \emph{does} exist, it may
not be \emph{recoverable} through the stochastic masking interface employed by
post-hoc methods such as LIME or KernelSHAP. The constraint in our setting is
not the Kolmogorov complexity of $g$, but the mutual information that the query
channel can convey about the latent explanation $\Phi$. By applying a strong
converse argument, we show that if the query budget lies below channel capacity,
reliable recovery is impossible regardless of the estimator, even when the true
explanation is sparse and perfectly well-defined.

\subsection{Contributions}

This paper provides a rigorous information-theoretic foundation for masking-based explanation methods. The main contributions are:

\begin{itemize}
    \item A formulation of explanation extraction as communication over a \emph{noisy query channel}, where explanation vectors act as messages and binary masks define channel inputs.
    \item A characterization of per-query capacity through mutual information $I(\Phi;Y \mid Z)$ and a corresponding explanation rate determined by the entropy of sparse supports.
    \item A strong converse theorem, adapted from Wolfowitz, proving that if the explanation rate exceeds channel capacity, no decoder can reliably recover the explanation regardless of computation.
    \item A Monte Carlo estimator that yields a nonasymptotic information limit $T_{\mathrm{IT}}$ and quantifies the gap between theory and the performance of LIME/Lasso and KernelSHAP/OLS.
    \item A source-coding interpretation of super-pixel resolution that predicts a critical segmentation beyond which explanations are information-theoretically infeasible.
\end{itemize}

\section{Background and Preliminaries}
\label{sec:background}

Masking-based explanation methods probe a black-box model through repeated evaluations under binary feature subsets. The purpose of this section is to establish a common technical foundation for analyzing these procedures through an information-theoretic lens. The focus is on the structure of the data generated by perturbed queries, the sparsity model governing the explanation, and the interpretation of the query mechanism as a channel.

\subsection{Perturbation-Based Explanations as Linear Inverse Problems}
\label{subsec:xai_sampling}

Local explanation methods such as KernelSHAP~\cite{lundberg2017shap} and LIME~\cite{ribeiro2016lime} operate by sampling binary masks around a fixed input $x \in \mathbb{R}^d$. For each query $t = 1,\ldots,T$, a mask $Z_t \in \{0,1\}^d$ is drawn from a sampling distribution~$\pi_Z$. The mask selects a subset of features from $x$, and the remaining coordinates are replaced by a baseline reference $x_{\mathrm{base}}$ (often zero or a dataset mean):
\begin{equation}
    x^{(t)} = x \odot Z_t + x_{\mathrm{base}} \odot (1 - Z_t),
    \label{eq:masked_input}
\end{equation}
where $\odot$ denotes elementwise multiplication. The black-box model is then evaluated on the perturbed sample, producing the scalar output
\[
    Y_t = f(x^{(t)}).
\]

These methods fit a local surrogate model to approximate the behavior of $f$ under such perturbations. A common representation is a linear function of the mask,
\begin{equation}
    g(Z_t) = \Phi_0 + Z_t^{\!\top}\Phi,
\end{equation}
where $\Phi \in \mathbb{R}^d$ defines the explanation and $\Phi_0$ is an offset. Estimating $\Phi$ from $\{(Z_t,Y_t)\}_{t=1}^T$ is a linear inverse problem: the masks act as measurement vectors and the outputs as observations. Differences between explanation methods arise from their sampling policy $\pi_Z$ and the loss or regularizer used during fitting (e.g., KernelSHAP’s weighted least squares or LIME’s sparse regression).

\subsection{Sparsity and the Support Recovery Objective}
\label{subsec:sparse_explanations}

Explanations are often expected to be sparse. Practical interpretability typically relies on identifying only a few influential features. We therefore assume that the latent explanation satisfies
\begin{equation}
    \|\Phi\|_0 = k \ll d,
\end{equation}
with support
\[
    S = \{ j : \Phi_j \neq 0 \}.
\]

The central inferential target is the support $S$, since it captures the discrete structure of the explanation. The goal of any explainer is to produce an estimate of this set; we denote such an estimate by $\widehat{S}$, obtained from the observations $\{(Z_t,Y_t)\}_{t=1}^T$ using any chosen decoding or regression procedure. Recovering $S$ from the query responses corresponds to a support recovery problem studied in high-dimensional statistics and information theory~\cite{wainwright2009}. 

To quantify reliability, we measure the probability that the estimated support differs from the true one,
\[
    P_e = \mathbb{P}(\widehat{S} \neq S).
\]
This metric depends on the explainer, the sampling strategy, and the underlying model, but the formulation itself does not assume any particular estimation method. It isolates the combinatorial component of the explanation—namely the identity of the active features—which is the part that can be characterized in terms of information content.

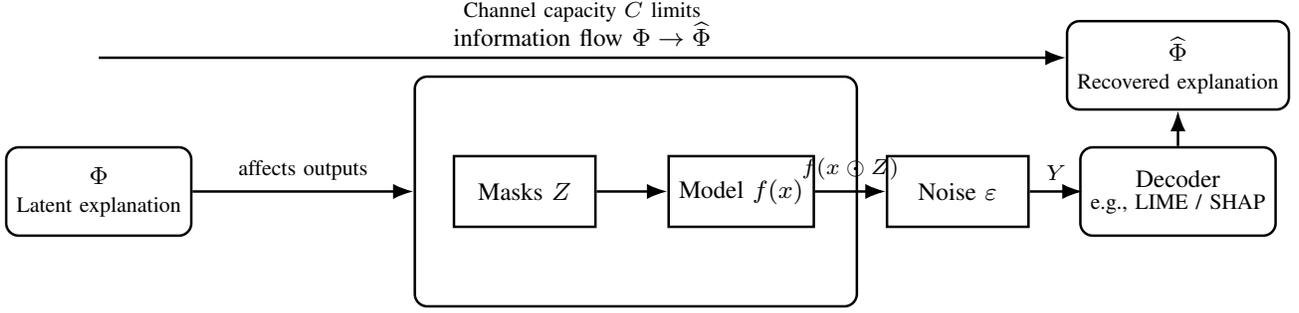
\begin{figure*}[t]
    \centering
    \resizebox{0.95\linewidth}{!}{%
    \begin{tikzpicture}[
        >=Latex,
        font=\footnotesize,
        block/.style = {draw, thick, rounded corners, align=center, minimum width=1.8cm, minimum height=1.0cm},
        op/.style    = {draw, thick, rectangle, align=center, minimum width=1.6cm, minimum height=0.8cm},
        line/.style  = {thick, -{Latex}},
        node distance = 1.6cm
    ]

    \node[block] (phi) {$\Phi$\\\scriptsize Latent explanation};

    \node[
        draw, thick, rounded corners,
        minimum width=5.0cm,
        minimum height=2.6cm,
        right=2.5cm of phi,
        align=center,
        inner sep=4pt
    ] (channelbox) {};


    \node[op] at ($(channelbox.west)!0.25!(channelbox.east)$) (masks) {Masks $Z$};
    \node[op, right=0.8cm of masks] (blackbox) {Model $f(x)$};
    \node[op, right=0.8cm of blackbox] (noise) {Noise $\varepsilon$};

    \node[block, right=2.5cm of channelbox] (decoder) {Decoder\\[-2pt]\scriptsize e.g., LIME / SHAP};
    \node[block, above=0.4cm of decoder] (phihat) {$\widehat{\Phi}$\\\scriptsize Recovered explanation};

    \draw[line] (phi) -- node[above, midway] {\scriptsize affects outputs} (channelbox.west);

    \draw[line] (masks) -- node[above, midway] {} (blackbox);
    \draw[line] (blackbox) -- node[above, midway] {\scriptsize $f(x \odot Z)$} (noise);

    \coordinate (chanOut) at ($(noise.east)+(0.6,0)$);
    \draw[line] (noise.east) -- node[above, midway] {\scriptsize $Y$} (chanOut);

    \draw[line] (decoder.north) -- (phihat.south);

    \coordinate (capL) at ($(phi.north)+(0,1.0)$);
    \coordinate (capR) at ($(phihat.west |- capL)$);
    \draw[line]
        (capL) --
        node[above, align=center] {\scriptsize Channel capacity $C$ limits\\information flow $\Phi \rightarrow \widehat{\Phi}$}
        (capR);

    \end{tikzpicture}
    }
    \caption{The Query Channel: The latent explanation $\Phi$ influences the model outputs observed under binary masks $Z$ through the noisy oracle, i.e., the black-box model, and noise, producing outputs from which a decoder attempts to reconstruct $\widehat{\Phi}$. The channel capacity $C$ fundamentally limits the information that can flow from $\Phi$ to $\widehat{\Phi}$.}
    \label{fig:concept_query_channel}
\end{figure*}

\subsection{The Query Process as a Communication Channel}
\label{subsec:query_channel_background}

The repeated perturbation–evaluation process can be viewed as a channel through which information about the latent explanation $\Phi$ is conveyed to the user. Each mask $Z_t$ plays the role of a channel input, and the corresponding model output $Y_t$ is the channel output. The pair $(Z_t,Y_t)$ depends on $\Phi$ through the masked input \eqref{eq:masked_input}. This induces a conditional distribution
\[
    P_{Y|Z,\Phi}(y\,|\,z,\phi),
\]
which describes how $\Phi$ influences the observable outputs for a given mask.

This interpretation does not assume that $f$ is linear. Any deviation of $f(x^{(t)})$ from a linear surrogate introduces a residual term, and any randomness in the model or numerical implementation introduces stochastic variation. Both effects increase the uncertainty in $Y_t$ given $(Z_t,\Phi)$ and therefore behave as noise from a channel perspective.

The masks are drawn from a sampling policy $\pi_Z$, which plays the role of an input distribution. Different explainers correspond to different policies; for example, LIME uses nearly uniform sampling, while KernelSHAP uses a distribution derived from cooperative game-theoretic axioms. The choice of $\pi_Z$ determines how informative each query is, but the structural form of the channel is the same across methods.

\subsection{Noise and Model Misfit}
\label{subsec:noise_background}

Two components contribute to uncertainty in the channel output:
\begin{enumerate}
    \item \textbf{Stochastic Noise ($\nu_t$):} Aleatoric uncertainty arising from the model itself (e.g., randomized smoothing, dropout layers) or the system (numerical precision).
    \item \textbf{Model Misfit / Interference ($\xi(x^{(t)}))$:} The deterministic residual $f(x^{(t)}) - g(Z_t)$ arising from the non-linearity of the black-box model. In the context of linear estimation, this curvature acts as \textit{structured interference}. Following standard Rate-Distortion practice \cite{cover2006elements}, we model this residual as additive Gaussian noise $\mathcal{N}(0, \sigma_{\text{misfit}}^2)$, which provides a worst-case lower bound on capacity.
\end{enumerate}

Both enter the channel model through the conditional distribution $P_{Y|Z,\Phi}$. For analytic convenience in later sections, these terms are often represented through an additive noise model
\begin{equation}
    Y_t = Z_t^\top \Phi + \underbrace{\xi(x^{(t)}) + \nu_t}_{\varepsilon_t},    \label{eq:additive_noise_background}
\end{equation}
with $\varepsilon_t$ capturing both stochastic and deterministic deviations.  This representation does not assume that the oracle is truly linear; it provides a tractable abstraction for quantifying the information content of the query responses.

\subsection{Information Measures}
\label{subsec:information_measures}

\begin{definition}[Entropy of the Support]
    The uncertainty in the $k$-sparse support is
\begin{equation}
    H(S)
    = \log_2 \binom{d}{k},
\end{equation}
measured in bits. 
\end{definition}
This quantity characterizes the combinatorial complexity of identifying which features matter. It reflects the ``information content'' of the explanation task, independent of any estimator.

\begin{definition}[Mutual Information]
    For the channel defined by $(Z_t,Y_t)$,
\begin{equation}
    I(\Phi;Y_t \mid Z_t)
\end{equation}
measures how informative a single query is about the explanation, conditioned on the chosen mask. Aggregating over $T$ queries yields
\begin{equation}
    I(\Phi;\mathbf{Y}^T \mid \mathbf{Z}^T),
\end{equation}
which quantifies the total information that the query process conveys about $\Phi$. 
\end{definition}
This quantity captures the ``information supplied'' by the interface.


\section{Information-Theoretic Limits of the Query Channel}
\label{sec:theory}

In this section we formalize the masking-based explanation problem as a communication problem over a noisy channel and state a strong converse theorem that governs the fundamental limits of recovery. Our treatment follows classical results of Wolfowitz~\cite{wolfowitz1957strong} but is adapted to the reversed roles of encoder and channel input that arise in post-hoc explainability.


\subsection{System Model}
\label{subsec:system_model}

To apply discrete coding theorems, we view the explanation recovery problem as determining which member of a discrete candidate set $\mathcal{M}$ generated the observations.

\begin{enumerate}
    \item \textbf{Message Set:} The explainer seeks to identify an index $M \in \{1, \dots, |\mathcal{M}|\}$. In the general case, $\mathcal{M}$ corresponds to a quantized lattice of explanation vectors. In the sparse case, $\mathcal{M}$ corresponds to the set of all possible feature support sets of size $k$. The entropy of the message is $H(M) = \log_2 |\mathcal{M}|$.
    
    \item \textbf{Query Strategy (Encoder):} A query strategy generates a sequence of binary masks $\mathbf{Z}^T = (Z_1, \dots, Z_T)$, where $Z_t \in \{0,1\}^d$. In non-adaptive methods (KernelSHAP, LIME), $Z_t$ are drawn i.i.d. from a design distribution $\pi_Z$.
    
    \item \textbf{Oracle (Channel):} Each query produces an observation $Y_t$ according to the channel law $P_{Y|Z,\Phi}$. We assume the channel is memoryless given the input and state:
    \begin{equation}
        P(\mathbf{Y}^T \mid \mathbf{Z}^T, \Phi) = \prod_{t=1}^T P(Y_t \mid Z_t, \Phi).
    \end{equation}
    
    \item \textbf{Decoder:} An estimator $\psi: \mathcal{Y}^T \times \mathcal{Z}^T \rightarrow \mathcal{M}$ produces a decoded estimate $\widehat{M}$. The probability of error is $P_e^{(T)} = \mathbb{P}(\widehat{M} \neq M)$.
\end{enumerate}


\subsection{Rate and Capacity}
\label{subsec:rate_capacity}

In the explanation setting, the unknown quantity is a latent parameter vector
$\Phi$ with a discrete component $M$ identifying the underlying hypothesis (for
example, a support set in the sparse case). Each explainer produces an
estimator $\widehat{M}$ based on a block of $T$ queries. To characterize the
fundamental limits of this identification problem, we introduce two quantities:
(i) the rate at which uncertainty about $M$ must be resolved, and (ii) the
maximum per-query information that the query channel can provide.

\begin{definition}[Explanation Rate]
\label{def:rate}
Let $M$ be a discrete random variable taking values in a finite set
$\mathcal{M}$ with entropy $H(M)$ bits. For a blocklength of $T$ queries, the 
\emph{explanation rate} is defined as
\begin{equation}
    R \triangleq \frac{H(M)}{T}
    \qquad \text{(bits per query)}.
\end{equation}
This quantity measures the average amount of information that must be acquired
per query in order to reliably determine the value of $M$.
\end{definition}

This definition does not assume a specific estimator or sampling policy; it
encapsulates the intrinsic information content of the identification problem
induced by the prior distribution of $M$.

\begin{definition}[Parameter-Identification Capacity]
\label{def:capacity}
For a single query using mask $Z \in \{0,1\}^d$, let the observable channel 
output be distributed according to $P_{Y|Z,\Phi}$. The 
\emph{parameter-identification capacity} of the query channel is defined as
\begin{equation}
    C^{(S)}
    \triangleq
    \sup_{\pi_Z}
    I(\Phi; Y \mid Z),
\end{equation}
where the supremum is taken over all input (mask) distributions $\pi_Z$ having 
full support on $\{0,1\}^d$. This quantity is the maximum mutual information 
between the latent parameter $\Phi$ and a single channel output $Y$ achievable 
under any admissible masking strategy.
\end{definition}

\begin{remark}
\textnormal{    The capacity defined above differs from Shannon’s classical channel    capacity, which concerns reliable transmission of arbitrary bit sequences. Here, the channel is used for \emph{parameter identification}: the latent parameter $\Phi$ induces the output distribution and must be inferred from observations.}
\end{remark}

\begin{remark}
    \textnormal{The supremum over $\pi_Z$ captures the fact that different explainer
    sampling policies correspond to different input distributions. Allowing all
    distributions with full support yields a channel property independent of
    any particular explainer.}
\end{remark}

\begin{remark}
    \textnormal{The quantity $I(\Phi;Y\mid Z)$ measures the information conveyed by a
    \emph{single} query. Over $T$ uses of a memoryless query channel,
    the maximum obtainable information grows at most linearly with $T$:}
    \[
        I(\Phi;\mathbf{Y}^T\mid \mathbf{Z}^T) \le T\, C^{(S)}.
    \]
\end{remark}


\subsection{Strong Converse Theorem}

The following result is a specialization of Wolfowitz’s strong converse for
discrete-time memoryless channels to the parameter-identification setting.

\begin{theorem}[Strong Converse for the Query Channel]
\label{thm:strongconverse}
Let $\mathcal{M}$ be a finite hypothesis set with $|\mathcal{M}|=2^{TR}$ and let
$M$ be uniformly distributed over $\mathcal{M}$. Consider any sequence of query
strategies $\pi$ (possibly adaptive) and any sequence of decoders $\psi$. If the
rate satisfies
\[
    R > C^{(S)},
\]
where $C^{(S)}$ is the parameter-identification capacity, then there exists a
constant $A>0$ (depending only on the channel law) such that the error
probabilities satisfy
\[
    P_e^{(T)} \ge 1 - e^{-AT},
\]
and consequently $\lim_{T\to\infty} P_e^{(T)} = 1$.
\end{theorem}

\begin{IEEEproof}[Proof Sketch]
Fix a sequence of query strategies and decoders. For each $m\in\mathcal{M}$, the
query process induces a product distribution
\[
    P_{Y^T|M=m,Z^T} = \prod_{t=1}^T P_{Y|Z,\Phi(m)},
\]
where $\Phi(m)$ is the parameter associated with message $m$. Since the channel
is memoryless, the log-likelihood ratio between any two hypotheses decomposes
into a sum of $T$ i.i.d. terms.

Wolfowitz’s strong converse for memoryless channels (see \cite{wolfowitz1957strong})
implies that if the transmission rate exceeds the single-use capacity, then the
average error probability of any decoder satisfies
\[
    P_e^{(T)} \ge 1 - \exp(-AT),
\]
for some $A>0$. It remains to identify the relevant per-use capacity. For any
masking distribution $\pi_Z$,
\[
    I(M;Y_t \mid Z_t)
    \le I(\Phi;Y_t\mid Z_t),
\]
because $M \to \Phi \to (Z_t,Y_t)$ forms a Markov chain. Taking the supremum
over input distributions yields
\[
    I(M;Y_t \mid Z_t) \le C^{(S)}.
\]
Summing over $t=1,\ldots,T$ gives
\[
    I(M;Y^T\mid Z^T) \le T C^{(S)}.
\]
If $R > C^{(S)}$, then $H(M)=TR$ exceeds the maximum mutual information that the
channel can convey. Wolfowitz’s exponent bound therefore forces
$P_e^{(T)} \to 1$ as $T \to \infty$.
\end{IEEEproof}

The theorem states that the query interface can convey a finite information per query. If the explanation task requires more than $C^{(S)}$ bits per
query, then no sequence of sampling policies or decoders—adaptive or
non-adaptive, randomized or deterministic—can prevent the error probability from
converging to one. This result is purely information-theoretic: it does not
depend on computational constraints, convexity assumptions, or properties of a
particular explainer such as LIME or SHAP. The threshold $R=C^{(S)}$ therefore
marks a fundamental identifiability boundary for any masking-based explanation
scheme.


\subsection{Instantiation for KernelSHAP: Linear Gaussian Channel}
\label{subsec:kernelshap_instantiation}

The purpose of this subsection is to illustrate how the general strong converse
(Theorem~\ref{thm:strongconverse}) applies when the black-box model admits a
local linear approximation with additive Gaussian noise. This is the regime
explicitly assumed by KernelSHAP, where the surrogate is fit via weighted least
squares under binary masks.

We consider the linear Gaussian oracle model
\begin{equation}
    Y_t = z_t^\top \Phi + \varepsilon_t, 
    \qquad \varepsilon_t \sim \mathcal{N}(0,\sigma^2),
    \label{eq:LGM_formal}
\end{equation}
with $z_t\in\{0,1\}^d$. The latent parameter $\Phi$ is not assumed sparse in
this subsection; this corresponds to dense explanations as used in standard
KernelSHAP implementations.

\subsubsection{Per-Query Mutual Information}

Conditioned on a mask $z_t$, the query-channel output obeys
\[
    Y_t \mid (Z_t=z_t,\Phi=\phi)
    \sim \mathcal{N}(z_t^\top\phi, \sigma^2).
\]
Thus, for fixed $z_t$, the mutual information between $\Phi$ and $Y_t$ satisfies
\begin{equation}
    I(\Phi;Y_t \mid Z_t = z_t)
    = \frac{1}{2}\log_2\!\left(
        1 + \frac{\mathrm{Var}(z_t^\top \Phi)}{\sigma^2}
      \right),
    \label{eq:per_query_MI_dense}
\end{equation}
where the variance is taken with respect to the prior over $\Phi$. 

Since the explainer may choose any input distribution $\pi_Z$ over masks,
the parameter-identification capacity from Definition~\ref{def:capacity} becomes
\begin{equation}
    C^{(S)} 
    = \sup_{\pi_Z} 
      \mathbb{E}_{Z\sim\pi_Z}\!\left[
         I(\Phi;Y\mid Z)
      \right].
    \label{eq:Cs_dense_def}
\end{equation}
For convenience, we define the single-mask upper bound
\[
    C_{\max}
    \triangleq
    \max_{z\in\{0,1\}^d}
    I(\Phi;Y \mid Z=z),
\]
so that
\begin{equation}
    C^{(S)} \le C_{\max}.
    \label{eq:Cs_le_Cmax_dense}
\end{equation}

This inequality does not assert that $C_{\max}$ equals the operational
capacity; its purpose is to provide a tractable envelope that can be used to
derive scaling statements in the dense regime.

\subsubsection{Rate Requirement Under Dense Explanations}

Under dense explanations, the message effectively specifies all $d$ entries
of $\Phi$ to resolution $\delta$ over a bounded range $[-B,B]$. Standard
quantization arguments show that the hypothesis set can be chosen so that
\[
    H(M) \approx d \log_2\!\left(\frac{B}{\delta}\right).
\]
Substituting this into the rate definition $R = H(M)/T$ gives
\[
    R \approx \frac{d}{T}\log_2\!\left(\frac{B}{\delta}\right).
\]

By Theorem~\ref{thm:strongconverse}, reliable identification requires
$R < C^{(S)}$. Using the envelope \eqref{eq:Cs_le_Cmax_dense}, a necessary
condition for reliable recovery is
\[
    \frac{d}{T}\log_2\!\left(\frac{B}{\delta}\right)
    < C_{\max}.
\]
Equivalently,
\begin{equation}
    T > \frac{d}{C_{\max}}
          \log_2\!\left(\frac{B}{\delta}\right).
    \label{eq:dense_threshold_formal}
\end{equation}

Expression \eqref{eq:dense_threshold_formal} provides an information-theoretic
explanation for the observed scaling of KernelSHAP: estimating all $d$ entries
of a dense explanation requires a number of queries that grows at least
linearly with $d$. This lower bound arises from the finite per-query
information budget imposed by the query channel.

\subsection{Sparse Explanations and LIME}
\label{subsec:sparse_instantiation}

We now apply the strong converse to the sparse setting, where only the support
of $\Phi$ must be identified. This corresponds to the behavior of sparse
estimators used in LIME (e.g., Lasso), which select a small subset of active
features.

Let $\Phi$ be $k$-sparse in $d$ dimensions, and let $S\subseteq[d]$ with
$|S|=k$ denote its support. As shown in Section~\ref{subsec:information_measures},
the entropy of the support is
\begin{equation}
    H(S)
    = \log_2 \binom{d}{k}
    \approx k\log_2\!\left(\frac{d}{k}\right),
    \label{eq:HS_sparse_formal}
\end{equation}
when $d$ is large and 
$k\ll d$.

Using the rate definition $R = H(S)/T$, the explanation rate becomes
\begin{equation}
    R_{\mathrm{sparse}}
    \approx 
    \frac{k}{T}\log_2\!\left(\frac{d}{k}\right).
    \label{eq:R_sparse_formal}
\end{equation}

Applying Theorem~\ref{thm:strongconverse} with the upper bound
$C^{(S)}\le C_{\max}$ yields the necessary condition for reliable support
identification:
\begin{equation}
    \frac{k}{T}\log_2\!\left(\frac{d}{k}\right)
    < C_{\max},
\end{equation}
or equivalently,
\begin{equation}
    T > \frac{k}{C_{\max}}
          \log_2\!\left(\frac{d}{k}\right).
    \label{eq:sparse_threshold_formal}
\end{equation}

Expression \eqref{eq:sparse_threshold_formal} aligns with the well-known
$T = \Omega(k\log(d/k))$ scaling for support recovery obtained in high-dimensional
statistics and compressed sensing. The novelty here is that the result follows
directly from the information-theoretic structure of the query channel, without
requiring restricted isometries or linear-algebraic arguments.

In contrast to the dense case \eqref{eq:dense_threshold_formal}, sparse recovery
requires dramatically fewer queries when $k\ll d$. This difference explains why
LIME, which leverages sparsity through $\ell_1$ regularization, succeeds in
regimes where dense least-squares estimators such as KernelSHAP are
information-limited.

\section{Achievability: Existence of a Reliable Decoder Below Capacity}
\label{sec:achievability}

This section establishes the achievability counterpart to the strong converse proved in the previous Section. Specifically, we show that whenever the explanation rate is below channel capacity, there exists a decoder---namely the maximum-likelihood (ML) sparse decoder---whose probability of error vanishes as the number of queries grows.


\subsection{Random Coding Interpretation of Mask-Based Sampling}
\label{subsec:randomcoding}

Achievability proofs in channel coding rely on ensembles of randomly generated codebooks whose codewords are sufficiently separated to allow reliable identification under noise. In the explanation-channel setting, no explicit codebook is designed; instead, the sampling policy automatically induces a random codebook.

Let $Z \in \{0,1\}^{T \times d}$ be the mask matrix formed by $T$ i.i.d.\ draws from a full-support distribution $\pi_Z$ (e.g., Bernoulli$(0.5)$). For a $k$-sparse hypothesis with support $S$, define the corresponding noiseless projection (codeword) as:
\begin{equation}
    \mathbf{c}_S = Z \Phi_S \in \mathbb{R}^T.
\end{equation}
The collection $\mathcal{C} = \{\mathbf{c}_S : S \subseteq [d], |S| = k\}$ constitutes a random codebook whose elements must be distinguishable by a decoder observing the noisy vector:
\begin{equation}
    \mathbf{Y} = \mathbf{c}_S + \boldsymbol{\varepsilon}, \qquad \boldsymbol{\varepsilon} \sim \mathcal{N}(0,\sigma^2 I_T).
\end{equation}

The structural properties required for achievability follow from standard concentration of measure results (specifically the Johnson-Lindenstrauss lemma and the Restricted Isometry Property):

\begin{enumerate}
    \item \textbf{Separation of Competing Hypotheses.}
    For any two distinct supports $S \neq S'$, define the difference vector $\Delta = \Phi_S - \Phi_{S'}$. This vector is non-zero on the symmetric difference of the supports. The Euclidean distance between codewords is:
    \begin{equation}
        \| \mathbf{c}_S - \mathbf{c}_{S'} \|^2 = \| Z \Delta \|^2.
    \end{equation}
    Because the rows of $Z$ are i.i.d.\ sub-Gaussian, the quadratic form $\| Z \Delta \|^2$ concentrates sharply around its expectation. As $T \to \infty$:
    \begin{equation}
        \| Z \Delta \|^2 \approx T \cdot \mathbb{E}[(z^\top \Delta)^2].
    \end{equation}
    Thus, with probability approaching one, all competing hypotheses generate codewords separated by a distance scaling with $\sqrt{T}$. Crucially, this argument holds regardless of whether supports $S$ and $S'$ overlap; it depends solely on the non-zero energy of the difference vector $\Delta$.

    \item \textbf{Asymptotic Orthogonality of Noise.}
    The noise vector $\boldsymbol{\varepsilon}$ is isotropic and independent of $Z$. In high dimensions, the inner product between $\boldsymbol{\varepsilon}$ and any fixed signal direction concentrates near zero:
    \begin{equation}
        \langle \boldsymbol{\varepsilon}, Z\Delta \rangle = o\left( \| Z \Delta \| \right) \quad \text{with high probability}.
    \end{equation}
    Geometrically, this ensures that the received vector $\mathbf{Y}$ remains closer to the true codeword $\mathbf{c}_S$ than to any imposter $\mathbf{c}_{S'}$.
\end{enumerate}

These properties replicate the geometry of a classical random codebook. Since the signal separation grows as $\sqrt{T}$ while noise fluctuations remain bounded, a Maximum Likelihood decoder over the $k$-sparse hypothesis class achieves vanishing error probability provided the explanation rate is below capacity. This establishes the achievability component of the query-channel reliability theory.

\subsection{Optimal Decoder: Maximum-Likelihood Support Reconstruction}
\label{subsec:MLdecoder}

The optimal decoder for the sparse explanation channel is the Maximum
Likelihood (ML) estimator restricted to the set of $k$-sparse hypotheses.
Given a candidate support $S' \subseteq [d]$ with $|S'|=k$, define the
least-squares estimate of the amplitudes on $S'$ as
\begin{equation}
    \Phi_{S'}^{\mathrm{LS}}
    = \arg\min_{\Phi : \operatorname{supp}(\Phi)\subseteq S'}
        \| Y - Z\Phi \|_2^2.
    \label{eq:ls_support}
\end{equation}
The corresponding predicted signal is $Z \Phi_{S'}^{\mathrm{LS}}$.
Under the Linear Gaussian Model, the ML decision rule is
\begin{equation}
   \widehat{S}_{\mathrm{ML}}
   = \arg\min_{\substack{S' \subseteq [d] \\ |S'|=k}}
        \| Y - Z \Phi_{S'}^{\mathrm{LS}} \|_2^2.
   \label{eq:ML_decoder}
\end{equation}

This rule coincides with minimum-distance decoding on the induced random
codebook $\{\mathbf{c}_{S'} = Z\Phi_{S'}^{\mathrm{LS}}\}$: the decoder selects
the hypothesis whose induced codeword is closest to the received vector in
Euclidean distance. Although evaluating \eqref{eq:ML_decoder} requires a
combinatorial search over all $\binom{d}{k}$ supports, achievability results
permit unbounded computation and therefore focus solely on the
information-theoretic feasibility of reliable recovery.

\subsection{Achievability Theorem}
\label{subsec:achievability_theorem}

Let $H(S)=\log_2 \binom{d}{k}$ denote the entropy of the uniformly distributed
$k$-sparse support, and let $C$ denote the per-query capacity of the
explanation channel defined in Section~\ref{sec:theory}. The explanation rate is
\[
    R = \frac{H(S)}{T} \quad \text{bits per query}.
\]

\begin{theorem}[Achievability of Reliable Sparse Explanation Recovery]
\label{thm:achievability}
Consider the Linear Gaussian explanation channel \eqref{eq:LGM_formal}, and let
the mask rows $\{Z_t\}$ be drawn i.i.d.\ from a distribution with full support
on $\{0,1\}^d$ and sub-Gaussian coordinates (e.g., Bernoulli$(1/2)$). If the
explanation rate satisfies
\[
    R < C,
\]
then the Maximum Likelihood decoder \eqref{eq:ML_decoder} satisfies
\[
    \mathbb{P}(\widehat{S}_{\mathrm{ML}} \neq S)
        \longrightarrow 0
    \qquad \text{as } T \to \infty,
\]
with high probability over the draw of the mask matrix $Z$ and the channel
noise.
\end{theorem}

This result is the direct analogue of Shannon’s Channel Coding Theorem: whenever
the rate lies below capacity, almost every realization of the random mask
ensemble yields a codebook for which ML decoding achieves vanishing error
probability. Combined with the strong converse, this establishes $R=C$ as the
fundamental boundary separating reliable and impossible explanation.

\subsection{Proof Outline}
\label{subsec:proof_outline}

The proof follows the standard random coding template. We summarize the main steps.

\paragraph*{1) Separation of hypotheses}
For any incorrect support $S' \neq S$, define the prediction error
\begin{equation}
    \Delta(S,S') 
    = \| Z(\Phi_S - \Phi_{S'}^{\mathrm{OLS}}) \|_2^2.
\end{equation}
Because the columns of $Z$ are random and independent across rows, the terms $Z(\Phi_S - \Phi_{S'})$ form sums of i.i.d.\ random variables. Standard concentration inequalities imply that for each fixed pair $(S,S')$,
\begin{equation}
    \frac{1}{T}\Delta(S,S') 
    \xrightarrow[T \to \infty]{\mathsf{P}}
    \mathbb{E}\big[(Z_1^\top(\Phi_S - \Phi_{S'}))^2\big].
\end{equation}
Moreover, with high probability, the similarity between distinct supports is bounded away from zero because $\Phi_S - \Phi_{S'}$ differs in at least one coordinate.

\paragraph*{2) Error event decomposition}
The ML decoder fails if there exists a false support $S'$ whose prediction error falls within the ``noise sphere'' surrounding the true signal $Z\Phi_S$. The probability of this event is bounded via a union bound over all $S' \neq S$:
\begin{equation}
    \mathbb{P}(\widehat{S}_{\mathrm{ML}} \neq S)
    \le \sum_{S' \neq S} 
    \mathbb{P}\!\left(
        \|Y - Z\Phi_{S'}^{\mathrm{OLS}}\|^2 
        \le
        \|Y - Z\Phi_S\|^2
    \right).
\end{equation}
Each term in the sum decays exponentially with $T$, with exponent governed by the divergence between the two Gaussian distributions $P(Y|S)$ and $P(Y|S')$.

\paragraph*{3) Capacity threshold}
For any $S' \neq S$, the relative entropy
\begin{equation}
    D\!\left(P(Y|S) \,\|\, P(Y|S')\right)
\end{equation}
is strictly positive. The aggregate number of hypotheses satisfies
\begin{equation}
    \log \binom{d}{k} = H(S) = RT.
\end{equation}
Using standard bounds for random coding error exponents, the union bound vanishes as $T\to\infty$ if and only if
\begin{equation}
    R < C,
\end{equation}
where $C$ is the per-query capacity of the channel, matching the characterization in Section~\ref{sec:theory}. This completes the proof.

\paragraph*{Remark (Posterior Concentration).}
Although the ML decoder establishes the information-theoretic limit, it is
computationally infeasible for large $d$. A useful conceptual interpretation is
that below capacity ($R<C$), the likelihood 
$\exp(-\|Y - Z\Phi_{S'}^{\mathrm{LS}}\|_2^2/(2\sigma^2))$ assigns 
exponentially higher weight to the true support than to competing hypotheses.
Thus, posterior sampling or approximate Bayesian procedures tend to 
concentrate around the true support when $R<C$ and become unstable or
multimodal when $R>C$. This provides intuition for why practical 
sampling-based optimizers exhibit a similar transition, even though the ML
decoder remains the formal benchmark.

\section{Sample Complexity of Explanation}
\label{sec:sample_complexity}

This section empirically evaluates the achievability threshold from
Theorem~\ref{thm:achievability} and compares it with the query budgets required
by practical decoders. The goal is to place the ML achievability transition,
the mutual-information benchmark $T_{\mathrm{IT}}$, and the behavior of
sparsity-aware and dense decoders on a single common scale.

\subsection{Experimental Setup}

We consider a low-dimensional configuration in which exhaustive ML decoding is
feasible. The explanation vector $\Phi \in \mathbb{R}^d$ is $k$-sparse, with
support $S$ drawn uniformly at random from all $\binom{d}{k}$ subsets. The
nonzero entries are independent $\mathcal{N}(0,1)$ and the remaining entries
are zero. Given a blocklength $T$, the query masks are drawn as
$Z_t \in \{0,1\}^d$ with i.i.d.\ Bernoulli components, and the oracle outputs
follow the Linear Gaussian Model
\[
    Y_t = Z_t^\top \Phi + \varepsilon_t, \qquad
    \varepsilon_t \sim \mathcal{N}(0,\sigma^2),
\]
independently across $t=1,\dots,T$.

The unified experiment uses $d = 12, k = 2,\sigma = 0.1,$
and evaluates a range of query budgets $T$. For each $T$ and each Monte Carlo trial, a fresh realization of
$(\Phi,\mathbf{Z}^T,\mathbf{Y}^T)$ is generated under this model.

Three decoders are applied to the same data:
\begin{enumerate}
    \item \textbf{Optimal ML decoder:} a combinatorial search over all supports
          of size $k$, solving a least-squares fit for each candidate support.
          This decoder attains the achievability limit in
          Theorem~\ref{thm:achievability}.
    \item \textbf{Lasso decoder (sparse):} an $\ell_1$-regularized regression
          on the masked design, returning the $k$ indices with largest absolute
          coefficients. This is the sparse decoder associated with LIME-type
          estimators.
    \item \textbf{OLS decoder (dense):} an ordinary least-squares regression
          on all $d$ features, followed by selecting the $k$ largest coefficients.
          This mirrors the dense regression step used in KernelSHAP.
\end{enumerate}

For each decoder we measure the empirical probability of exact support recovery
    $\Pr(\widehat{S} = S)$,
estimated over many independent draws of $(\Phi,Z^T,Y^T)$ at each $T$.


The information-theoretic query benchmark is defined by
\begin{equation}
    I(\mathbf{\Phi};\mathbf{Y}^T \mid \mathbf{Z}^T) \;\ge\; H(S),
    \label{eq:TIT_condition_unified}
\end{equation}
where the support entropy under the uniform-sparsity prior is
    $H(S) = \log_2 \binom{d}{k}$.
The sparse prior on $\mathbf{\Phi}$ does not yield a closed-form expression for
$I(\mathbf{\Phi};\mathbf{Y}^T \mid \mathbf{Z}^T)$, so we estimate it as $\widehat{I}_T$ via Monte Carlo integration using the
same oracle prior that generates the experiments\footnote{For each blocklength $T$:
\begin{enumerate}
    \item Draw $N_{\mathrm{outer}}$ independent triples
          $(\Phi_i, Z_i^T, Y_i^T)$ from the generative model.
    \item Pre-generate $N_{\mathrm{inner}}$ i.i.d.\ samples
          $\{\Phi_j\}_{j=1}^{N_{\mathrm{inner}}}$ from the same sparse prior.
\end{enumerate}

For each outer sample $i$, the conditional likelihood
$p(Y_i^T \mid Z_i^T,\Phi_i)$ under Gaussian noise has a closed-form expression.
The marginal $p(Y_i^T \mid Z_i^T)$ is approximated by averaging the same
Gaussian likelihood over the inner prior samples:
\[
    p(Y_i^T \mid Z_i^T)
    \approx
    \frac{1}{N_{\mathrm{inner}}}
    \sum_{j=1}^{N_{\mathrm{inner}}}
    p(Y_i^T \mid Z_i^T,\Phi_j).
\]

This yields the importance-sampling estimator
\begin{equation}
    \widehat{I}_T
    =
    \frac{1}{N_{\mathrm{outer}}}
    \sum_{i=1}^{N_{\mathrm{outer}}}
    \log_2
    \left(
        \frac{p(Y_i^T \mid Z_i^T,\Phi_i)}
             {\frac{1}{N_{\mathrm{inner}}} \sum_{j=1}^{N_{\mathrm{inner}}}
              p(Y_i^T \mid Z_i^T,\Phi_j)}
    \right).
    \label{eq:MI_estimator_unified}
\end{equation}
}.
The operational information threshold is then
\begin{equation}
    T_{\mathrm{IT}}
    \triangleq
    \min\{T : \widehat{I}_T \ge H(S)\}.
    \label{eq:TIT_def_unified}
\end{equation}


\begin{figure}[t]
    \centering
    \includegraphics[width=0.95\linewidth]{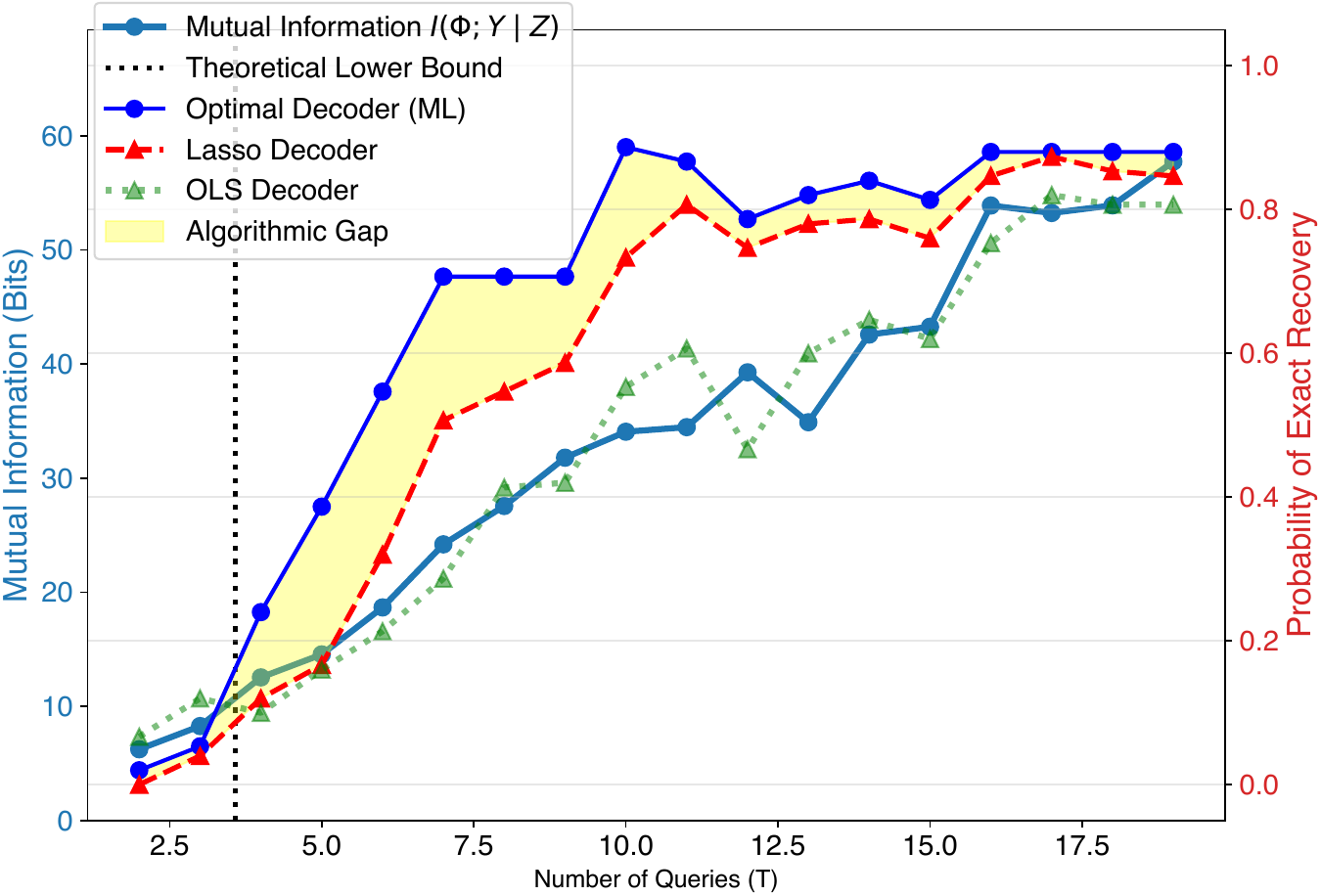}
    \caption{Unified information-theoretic and algorithmic behavior for the
    sparse explanation model with $d=12$, $k=2$, and $\sigma=0.1$.
    Left axis: Monte Carlo estimate of $I(\Phi;Y^T \mid Z^T)$ and a vertical
    line at $T = k \log(d/k)$, indicating the analytic lower bound for sparse
    recovery. Right axis: empirical probability of exact support recovery for
    the ML, Lasso, and OLS decoders as a function of the number of queries $T$.
    The horizontal separation between the ML and Lasso curves, highlighted by
    the shaded region, represents the algorithmic gap between the information
    limit and the performance of a convex surrogate decoder.}
    \label{fig:unified_achievability}
\end{figure}
\subsection{Achievability Transition and Algorithmic Gap}

Figure~\ref{fig:unified_achievability} summarizes the empirical behavior under this configuration. The left axis plots the estimated
mutual information $\widehat{I}_T$ against the number of queries $T$, together
with a vertical line at $T = k \log(d/k)$ marking the analytic lower bound for
sparse recovery. The curve $\widehat{I}_T$ grows with $T$ and crosses the
support entropy $H(S)$ near this scale, which identifies the information
benchmark $T_{\mathrm{IT}}$.

\begin{figure*}[t]
\centering
\includegraphics[width=0.9\linewidth]{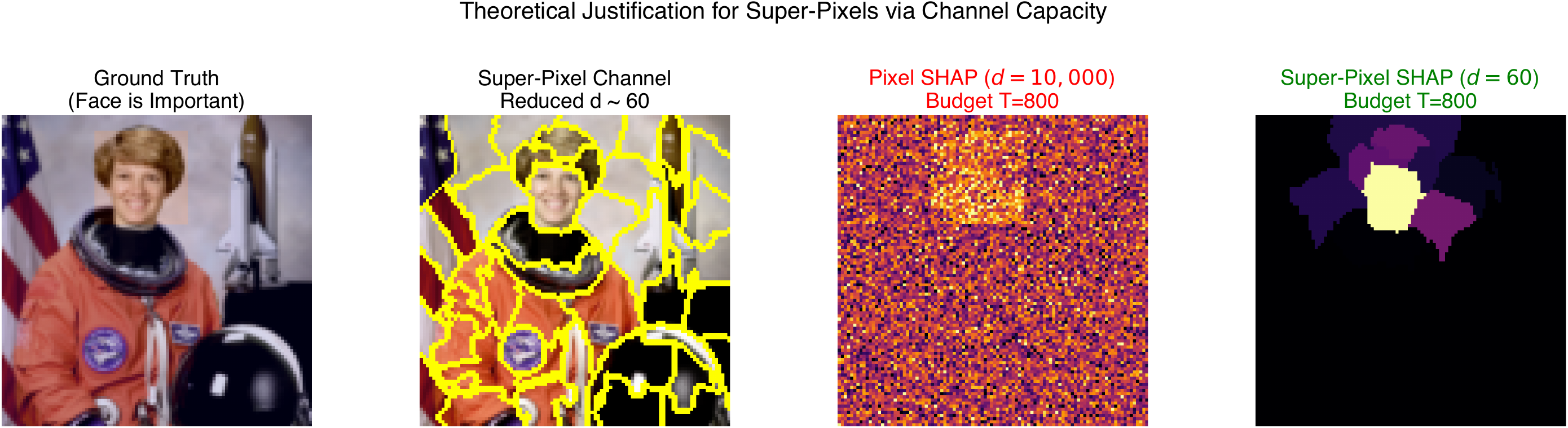}
\includegraphics[width=0.9\linewidth]{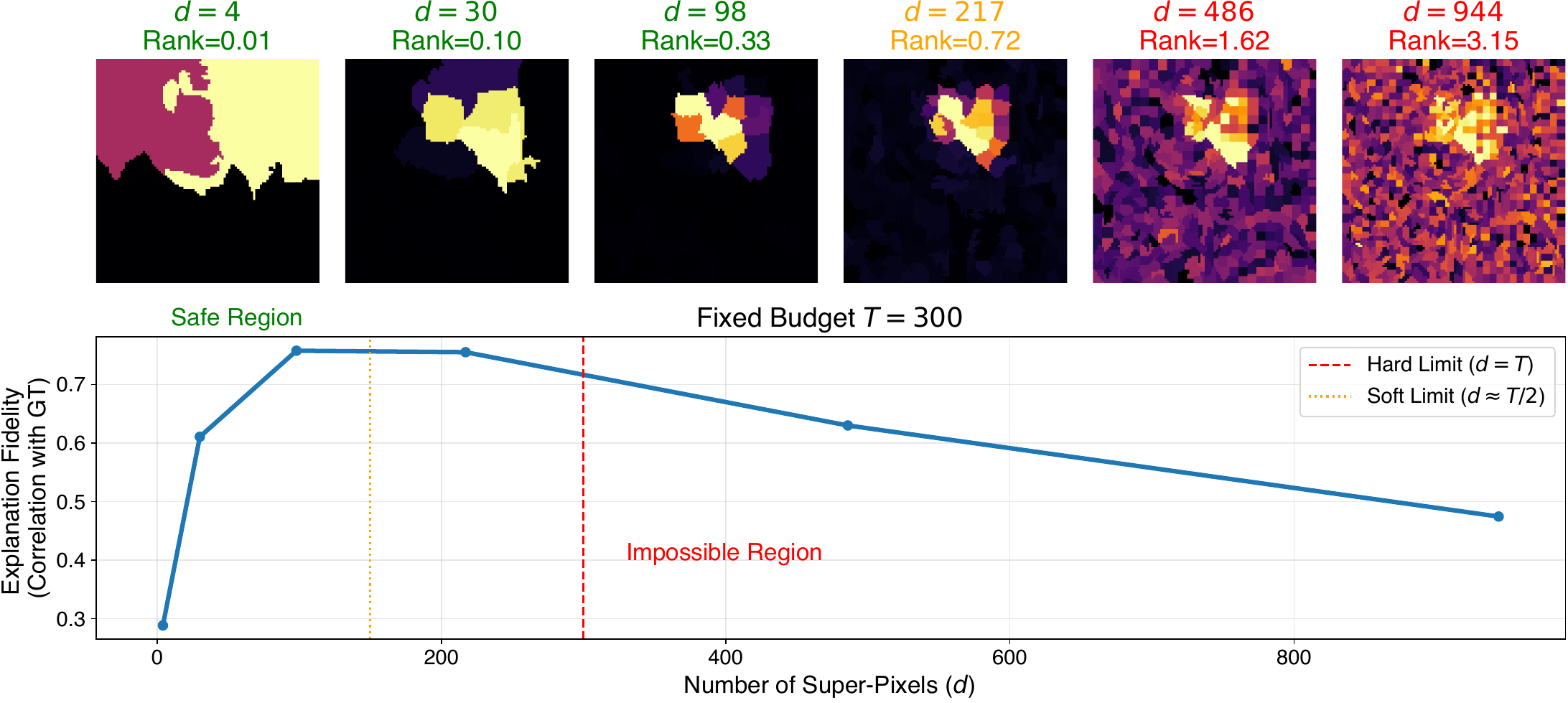}
\caption{Sample complexity at fixed query budget and varying resolution. Top: conceptual illustration comparing pixel-level SHAP at very high dimension $d\approx 10^4$ with a super-pixel channel at $d\approx 60$ under the same query budget $T=800$. Pixel-level explanations operate at a rate $R>C$ and are unrecoverable, while super-pixel explanations operate at $R<C$ and are recoverable. Bottom: quantitative capacity breakdown for a fixed budget $T=300$. The curve shows the correlation between the Lasso super-pixel explanation and the ground truth as a function of the number of super-pixels $d$. The safe region corresponds to $H(S)\le I(\Phi;Y^T \mid Z^T)$, whereas the impossible region corresponds to $H(S)>I(\Phi;Y^T \mid Z^T)$.}
\label{fig:capacity_breakdown}
\end{figure*}

The right axis shows the empirical probability of exact support recovery
$\Pr(\widehat{S}=S)$ for the three decoders. The ML decoder exhibits a sharp
transition from failure to success at query budgets close to
$T_{\mathrm{IT}}$, in line with Theorem~\ref{thm:achievability}. In contrast,
the Lasso and OLS decoders require more queries to reach high success
probability, with OLS performing worst due to its dense parameterization.

The region between the Lasso and ML curves in
Fig.~\ref{fig:unified_achievability} defines an \emph{algorithmic gap}: for
$T$ in this interval, information theory allows reliable recovery, but the
convex surrogate used by Lasso still fails with non-negligible probability.
This gap quantifies the relaxation penalty between the Shannon limit and
practical decoding in masking-based explanations. The OLS curve shows that
ignoring sparsity leads to even larger query budgets, consistent with the
ambient-dimension scaling of dense regression.

\subsection{Sample Complexity at Fixed Budget and Varying Resolution}

The previous experiment varied the number of queries $T$ at fixed dimension $d$.
To connect more directly with image-based explanations, we now fix the query
budget $T$ and vary the explanation resolution $d$ through the number of
super-pixels. In this regime, the question is how finely an image can be
segmented before a given query budget becomes insufficient for reliable
explanation.

We use a standard image (\texttt{scikit-image} ``Astronaut'') and define a
ground-truth salient region (the face) as a rectangular mask at the pixel
level. This mask induces a sparse support $S$ in the pixel domain: pixels in
the face region are relevant, all others are irrelevant. For each choice of
super-pixel resolution $d$, we apply a SLIC segmentation \cite{achanta2012slic} with $d$ regions and
map the pixel-level ground truth to a coarse explanation by averaging the
pixel-level importance within each super-pixel. This defines a $d$-dimensional
ground-truth super-pixel explanation.

For each value of $d$ we then:
\begin{enumerate}
    \item Generate $T$ binary masks over the $d$ super-pixels, forming a mask
          matrix $Z^T \in \{0,1\}^{T \times d}$.
    \item Map each super-pixel mask to the image domain and query the
          image-level oracle under these masked inputs to obtain $Y^T$.
          The oracle sums intensities within the face region and adds Gaussian
          noise, implementing a simple but controlled black-box model.
    \item Fit a linear regression decoder (Ridge) \cite{hoerl1970ridge} from the super-pixel masks to
          the oracle outputs, and project the recovered coefficients back to
          the image domain to obtain an estimated explanation map
          $\widehat{\Phi}$.
    \item Measure the explanation quality by the correlation between
          $\widehat{\Phi}$ and the ground-truth super-pixel explanation.
\end{enumerate}

The top panel of Fig.~\ref{fig:capacity_breakdown}  provides a conceptual illustration at a fixed
budget $T=800$. The first image shows the original astronaut with the face
region overlaid as the pixel-level ground truth. The second image shows a
super-pixel segmentation with $d \approx 60$, which compresses the explanation
into a much lower-dimensional representation. The third panel computes a
pixel-level SHAP-style explanation treating each pixel as a feature
($d \approx 10^4$) under the same budget $T=800$; in this regime, the rate
$R = H(S)/T$ exceeds the channel capacity, and the resulting heatmap fails to
localize the face. The fourth panel repeats the experiment using a
super-pixel SHAP channel with $d \approx 60$; here $R<C$, and the recovered
heatmap now aligns with the face region.

The middle panel of Fig.~\ref{fig:capacity_breakdown} makes this transition quantitative at a fixed
budget $T=300$. The top row within this panel shows reconstructed explanation
heatmaps for increasing numbers of super-pixels $d$. Each subplot is annotated
with the actual $d$ and the ratio $d/T$, which serves as an effective
resolution-to-budget ratio. For small $d$, the face region is
clearly visible and the decoder recovers a stable explanation. As $d$ grows,
the same budget is spread over more degrees of freedom, the ratio $d/T$
increases, and the reconstructed maps become noisy and diffuse.

The curve at the bottom of the panel plots the empirical correlation between
the decoded explanation and the ground truth as a function of $d$ for
$T=300$. Two vertical reference lines mark a ``soft'' limit near $d \approx
T/2$ and a ``hard'' limit near $d \approx T$. For $d$ well below $T/2$, the
correlation remains high and corresponds to
resolutions where the query channel can still support accurate explanations.
As $d$ approaches and exceeds these reference scales, the correlation drops
rapidly and eventually collapses toward zero. 

From the information-theoretic viewpoint, as the
super-pixel resolution increases, the support entropy $H(S)$ grows with $d$,
while the mutual information $I(\Phi;\mathbf{Y}^T \mid \mathbf{Z}^T)$ is bounded by the fixed
budget $T$ and noise level. There exists a critical resolution at which
$H(S)$ exceeds $I(\Phi;\mathbf{Y}^T \mid \mathbf{Z}^T)$, beyond which reliable explanation is
impossible regardless of the decoder. Section~\ref{sec:resolution_source_coding}
formalizes this effect as a source-coding problem and extends the analysis to
real-image tasks using the same Astronaut oracle.

\subsection{Empirical Behavior under Noise and Curvature}
\label{subsec:noise_curvature}

\begin{figure}[t]
    \centering
    \includegraphics[width=0.95\linewidth]{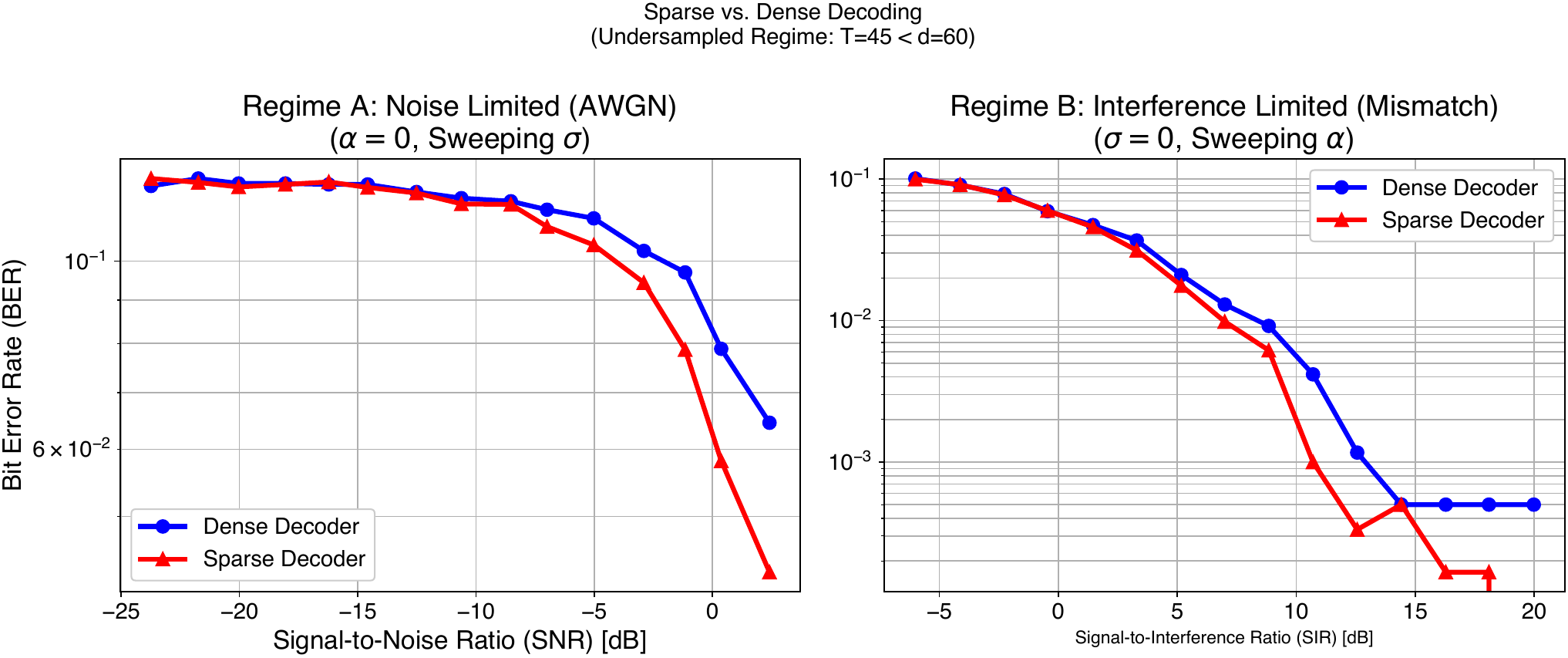}
    \caption{Block error probability for dense (Ridge / OLS proxy) and sparse (Lasso proxy) decoders in an undersampled regime $T<d$. Left: noise-limited regime with additive Gaussian noise and zero curvature. Right: interference-limited regime with deterministic quadratic curvature and zero noise. The horizontal axis is the empirical signal-to-distortion ratio (noise or interference).}
    \label{fig:waterfall}
\end{figure}

We now study how additive noise and model curvature affect support recovery when the number of queries $T$ is fixed below the ambient dimension $d$. The goal is to interpret the empirical “waterfall” and “error floor” behavior in Fig.~\ref{fig:waterfall} through the information-theoretic bound
\[
    I(\Phi;Y^T \mid Z^T) \;\ge\; H(S)
\]
that underlies Theorem~\ref{thm:strongconverse}. Throughout this subsection, the true explanation is represented by a sparse coefficient vector $\Phi^\star \in \mathbb{R}^d$ with support $S \subset [d]$, $\lvert S\rvert = k$, and nonzero entries $\pm 1$.

For each query index $t=1,\dots,T$, the explainer selects a binary mask $Z_t \in \{0,1\}^d$ and observes a scalar response
\begin{equation}
    Y_t
    = Z_t^\top \Phi^\star
      + \alpha\, Z_t^\top Q Z_t
      + \varepsilon_t,
    \label{eq:curved_channel}
\end{equation}
where $Q \in \mathbb{R}^{d \times d}$ is a fixed symmetric interaction matrix, $\alpha \ge 0$ is a curvature (interference) coefficient, and $\varepsilon_t \sim \mathcal{N}(0,\sigma^2)$ are i.i.d.\ Gaussian noise samples. The sequence $\{Z_t\}_{t=1}^T$ is drawn i.i.d.\ from a product distribution on $\{0,1\}^d$ that does not depend on $\Phi^\star$.

The linear-Gaussian model of Section~\ref{sec:background} is recovered as the special case $\alpha = 0$, $\sigma > 0$; however, when $\alpha>0$ and $\sigma=0$, the observation includes deterministic, mask-dependent interference:
\[
    Y_t = Z_t^\top \Phi^\star + \alpha\, Z_t^\top Q Z_t,
\]
and the conditional law of $Y_t$ given $(Z_t,\Phi^\star)$ is no longer Gaussian and no longer linear in $\Phi^\star$.

Given $T$ queries and responses $\{(Z_t,Y_t)\}_{t=1}^T$, we form two linear estimators:
a dense decoder $\widehat{\Phi}^{\mathrm{dense}}$ obtained by ridge regression and a sparse decoder $\widehat{\Phi}^{\mathrm{sparse}}$ obtained by Lasso. In both cases, the estimated support is defined by selecting the $k$ largest coordinates in magnitude:
\[
    \widehat{S} = \operatorname*{arg\,top\text{-}k}_{j \in [d]}
    \bigl\lvert \widehat{\Phi}_j \bigr\rvert.
\]
We measure performance via the block error probability $ P_{\mathrm{err}}
    \;=\;
    \mathbb{P}\bigl[\widehat{S} \neq S\bigr],$
estimated empirically by averaging over independent draws of $\Phi^\star$, the masks $\{Z_t\}$, and the noise $\{W_t\}$.

\paragraph*{Noise-limited regime}
In the first family of experiments (left panel of Fig.~\ref{fig:waterfall}), we fix $\alpha = 0$ and vary the noise variance $\sigma^2$. For each $\sigma^2$, we compute the empirical signal and noise powers
\[
    P_S = \frac{1}{T}\sum_{t=1}^T (Z_t^\top \Phi^\star)^2,
    \qquad
    P_N = \frac{1}{T}\sum_{t=1}^T \varepsilon_t^2,
\]
and report the block error probability as a function of the effective signal-to-noise ratio $\mathrm{SNR} = P_S/P_N$ (in dB). Both decoders display threshold behavior: as $\mathrm{SNR}$ increases, $P_{\mathrm{err}}$ drops rapidly from values near $1$ to values near $0$.

From an information-theoretic viewpoint, increasing $\sigma^2$ decreases the per-query mutual information $I(\Phi;Y_t \mid Z_t)$ and hence the total information $I(\Phi;\mathbf{Y}^T \mid \mathbf{Z}^T)$ for fixed $T$. For large noise, this quantity is below $H(S)$ and Theorem~\ref{thm:strongconverse} predicts high block error for any decoder, which is consistent with the plateau at $P_{\mathrm{err}}\approx 1$. As $\sigma^2$ decreases, $I(\Phi;\mathbf{Y}^T \mid \mathbf{Z}^T)$ increases and eventually exceeds $H(S)$; in this regime, sparse decoding reaches low error at a lower effective $\mathrm{SNR}$ than dense decoding, reflecting the fact that the sparse decoder targets the lower-entropy support $S$ rather than the full ambient vector.

\paragraph*{Interference-limited regime}
In the second family of experiments (right panel of Fig.~\ref{fig:waterfall}), we fix $\sigma = 0$ and vary the curvature coefficient $\alpha$. The interaction matrix $Q$ is scaled so that, for $\alpha=1$, the average interference power matches the average linear signal power. For each $\alpha$, we decompose the response into a linear component and a curvature-induced component,
\[
S_t = Z_t^\top \Phi^\star,
\qquad
I_t = \alpha, Z_t^\top Q Z_t,
\]
estimate empirical powers $P_S = \mathbb{E}[S_t^2]$ and $P_I = \mathbb{E}[I_t^2]$, and report an effective signal-to-interference ratio $\mathrm{SIR} = P_S/P_I$ on the horizontal axis (in dB). For small $\alpha$ (high $\mathrm{SIR}$), the channel is close to the linear model and the sparse decoder retains a clear block-error advantage over the dense decoder. As $\alpha$ increases and the quadratic term becomes comparable to, and eventually larger than, the linear term, both decoders develop an error floor: even at moderate or high $\mathrm{SIR}$, the block error probability does not approach zero.

In this regime, the dominant degradation mechanism is model mismatch rather than stochastic noise. The quadratic term $Z_t^\top Q Z_t$ induces structured, mask-dependent distortions that cannot be represented by any coefficient vector $\Phi$ in the linear surrogate class. Information-theoretically, this implies that $I(\Phi;Y^T \mid Z^T)$ saturates as the signal power $\lVert \Phi^\star \rVert_2^2$ increases, and can remain below $H(S)$ for the chosen $(d,k,T)$. The error floor in Fig.~\ref{fig:waterfall} is therefore consistent with Theorem~\ref{thm:strongconverse}: once $I(\Phi;Y^T \mid Z^T) < H(S)$, no linear decoder can drive $P_{\mathrm{err}}$ to zero, even if the nominal $\mathrm{SIR}$ (computed from $P_S/P_I$) is large.


\section{Resolution as Source Coding}
\label{sec:resolution_source_coding}

The preceding sections established the query interface as a communication channel with per-query capacity $C^{(S)}$ and showed that support recovery is feasible only when the explanation rate $R$ is strictly below this capacity. This section formalizes the dependence of $R$ on the \emph{resolution} used for image explanations and shows that super-pixel granularity induces a source-coding constraint that directly determines the feasibility of reliable recovery. We have also provided extension of resolution as source coding to text processing over tokenization from explainability lens.

Throughout, an image is partitioned into $d$ disjoint super-pixels, and the explanation is modeled as a $k$-sparse vector $\Phi \in \mathbb{R}^d$ with support $S \subseteq [d]$, $|S| = k$. The segmentation level $d$ defines the size of the discrete message alphabet and therefore the entropy that must be conveyed through the query channel.

\subsection{Source Entropy Induced by Resolution}

Under a uniform prior over all $k$-subsets of $[d]$, the discrete message is the support index $S$. Its entropy is
\begin{equation}
    H(S(d)) 
    = \log_2 \binom{d}{k}
    \approx k \log_2\!\left( \frac{d}{k} \right),
    \label{eq:HS_resolution_formal}
\end{equation}
which grows monotonically in $d$. For any fixed $k$, increasing resolution increases the cardinality of the support alphabet and therefore the number of bits that must be resolved by the decoder.

Given a query budget $T$, the induced explanation rate is
\begin{equation}
    R(d)
    \triangleq \frac{H(S(d))}{T}.
    \label{eq:rate_resolution_formal}
\end{equation}
This is the rate at which the support message must be transmitted per query.

\subsection{Information conveyed by the Query Channel}

Let $Y^T$ denote the $T$ oracle outputs generated under masks $Z^T$ drawn i.i.d. from a full-support design distribution. The mutual information
\begin{equation}
    I_T(d)
    \triangleq I(\Phi; Y^T \mid Z^T)
\end{equation}
quantifies the total information the query channel is capable of conveying. Because the number of queries $T$ and the mask distribution are fixed, $I_T(d)$ depends only on noise, model curvature, and mask–signal distinguishability, and is essentially constant in $d$ for typical image settings. Consequently, the denominator in \eqref{eq:rate_resolution_formal} is fixed while the numerator grows with $d$.

\subsection{Resolution-Constrained Feasibility}

The strong converse established in Theorem~\ref{thm:strongconverse} implies that recovery is impossible whenever the attempted transmission rate exceeds the capacity of the query channel:
\begin{equation}
    R(d) > C^{(S)}
    \quad \Longrightarrow \quad
    \lim_{T\to\infty} P_e^{(T)} = 1.
    \label{eq:resolution_converse_condition}
\end{equation}
Combining \eqref{eq:rate_resolution_formal} and \eqref{eq:resolution_converse_condition} yields the following resolution threshold.

\begin{definition}[Critical Resolution]
For fixed $(T, k)$ and channel capacity $C^{(S)}$, the \emph{critical resolution} is
\begin{equation}
    d_{\mathrm{crit}}
    \triangleq 
    \max\Big\{
        d \in \mathbb{N} \;:\; H(S(d)) \le T C^{(S)}
    \Big\}.
    \label{eq:dcrit_formal}
\end{equation}
\end{definition}

\begin{theorem}[Resolution-Constrained Strong Converse]
\label{thm:resolution_converse}
Let $d_{\mathrm{crit}}$ be given by \eqref{eq:dcrit_formal}.  
If $d > d_{\mathrm{crit}}$, then for any sequence of query strategies and any decoder,
\[
P_e^{(T)} \;\ge\; 1 - e^{-A T},
\]
for some constant $A>0$ depending only on the channel law. In particular,
\[
\lim_{T\to\infty} P_e^{(T)} = 1.
\]
Thus, reliable explanation at resolution $d$ is information-theoretically impossible.
\end{theorem}

The theorem follows directly from Theorem~\ref{thm:strongconverse} by substituting the rate $R(d)$ and observing that $R(d) > C^{(S)}$ exactly when $d > d_{\mathrm{crit}}$.

\subsection{Empirical Verification on Image-Based Explanations}

\begin{figure}[t]
    \centering
    \includegraphics[width=0.95\linewidth]{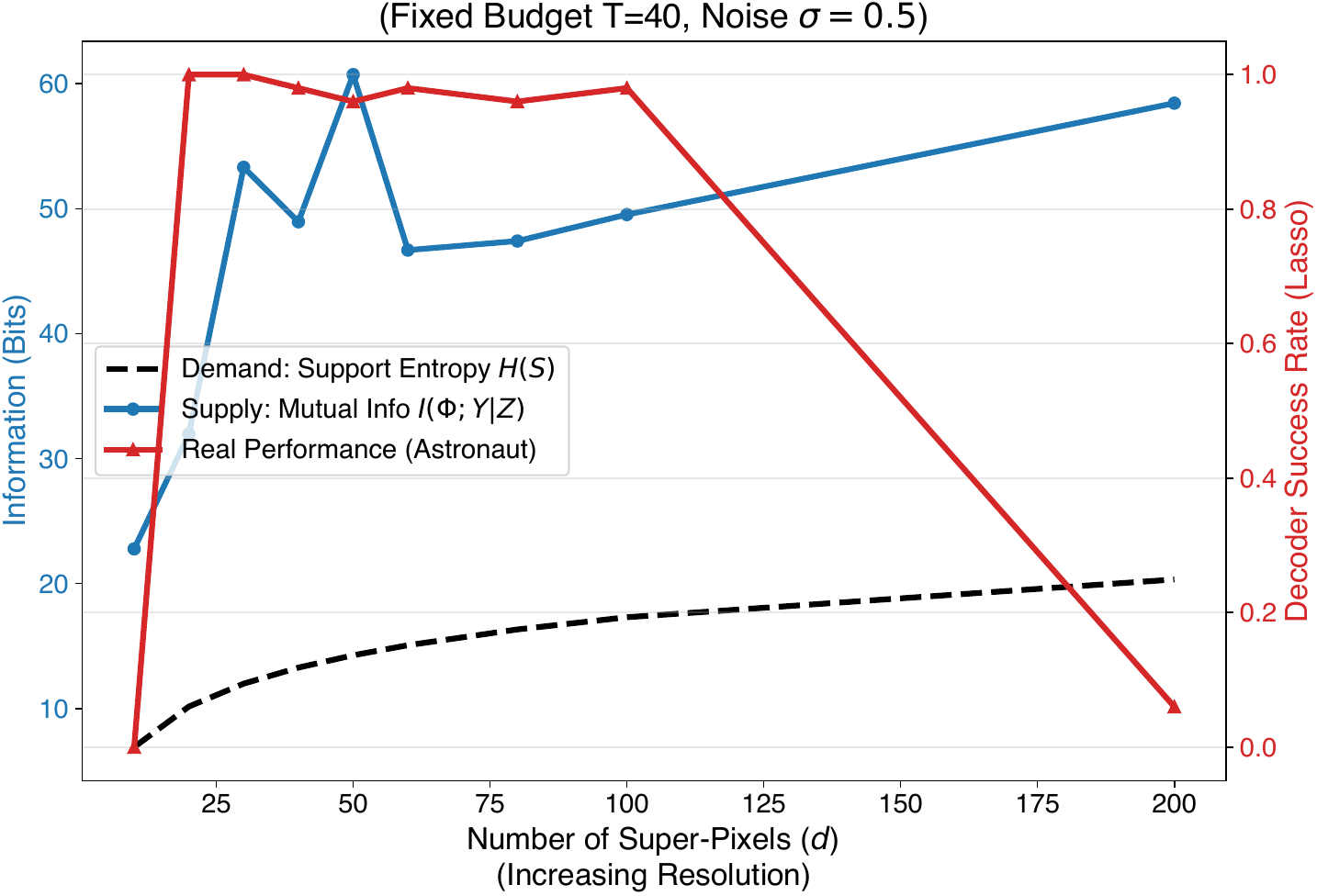}
    \caption{Empirical relationship between information content $H(S(d))$, information conveyed by the query channel $I_T(d)$, and decoder performance as a function of the number of super-pixels $d$.}
    \label{fig:superpixel_IT}
\end{figure}

This subsection examines whether the rate–capacity framework developed above accurately predicts the behavior of image-based explanations under varying resolution. The experiment uses the ``Astronaut'' image and applies a systematic sweep over the number of super-pixels $d$. For each value of $d$, the image is segmented into $d$ regions, and the ground-truth pixel-level explanation is aggregated onto this representation. A fixed number of queries $T$ is then applied, each consisting of a randomly drawn binary mask over the $d$ super-pixels. The oracle output for each mask is generated under the Linear Gaussian Model, and a Monte Carlo mutual-information estimator is used to compute $I_T(d)=I(\Phi;Y^T \mid Z^T)$ following the procedure described in Section~\ref{sec:sample_complexity}. A LIME/Lasso decoder is finally applied to obtain an estimated explanation, which is evaluated by correlation against the ground-truth super-pixel explanation.

Figure~\ref{fig:superpixel_IT} shows the resulting information accounting and algorithmic performance. The entropy $H(S(d))=\log_2\binom{d}{k}$ increases monotonically with $d$, reflecting the rising combinatorial complexity of identifying a $k$-sparse support in a $d$-dimensional representation. In contrast, the information supplied by the query channel, quantified by $I_T(d)$, remains approximately constant for fixed $T$, because the number of observations and the masking distribution are unchanged. The empirical correlation of the Lasso decoder with the ground truth follows these two curves closely: correlation remains high for all $d$ such that $H(S(d)) \le I_T(d)$ and deteriorates sharply once $H(S(d))$ exceeds $I_T(d)$.

This behavior matches the strong-converse threshold predicted by the theory. When $H(S(d)) > I_T(d)$, the explanation rate $R(d)=H(S(d))/T$ necessarily exceeds the effective capacity of the query channel. In this regime the strong converse guarantees that the probability of identifying the correct support converges to zero for any decoder, independent of its computational features or regularization strategy. Thus, the observed failure at higher resolutions is an information-theoretic inevitability rather than a consequence of estimator design.

Interpreting resolution through this lens formalizes its operational meaning. The choice of $d$ determines the cardinality of the support alphabet and therefore the number of bits that must be conveyed about the explanation. Increasing $d$ increases this informational requirement, and recovery remains feasible only while the channel can support a corresponding amount of information. The empirical transition in Fig.~\ref{fig:superpixel_IT} therefore constitutes a direct validation of the rate–capacity law for explanation channels.

\subsection{Tokenization as Source Coding and TokenSHAP}
The purpose of this subsection is not to reinterpret tokenization as an explanation problem. Rather, we argue that the source coding viewpoint developed for explanation resolution may extend more broadly to representation design in neural language models.
In neural language processing (NLP), model is considered as a function that predicts upcoming words, assign probability to entire sentence by using tokenizer and contextual representation of token embeddings. Tokenizer consist tuple of encoder, decoder and finite dictionary where encoder is splitting the input text into smaller units \cite{rajaraman2024analysis}. Thanks to encoding in the tokenizer, model is able to quantize the input and process them in a computational way properly. At this point, how to split input text, i.e segmentation, directly influence on model's performance either in downstream tasks or pretraining phase \cite{why_not_char},\cite{kudo2018sentencepiece}, \cite{byt5}. For example, oversegmentation of the text requires much more memory complexity and training time while the input is represented in a high grained-resoluted way \cite{bpe}, \cite{byt5}. On the other hand, undersegmentation is creating the scheme that tokens are becoming long and rare which brings huge loss, i.e low performance, to model. Even if low tokenizer fertility is expected as necessary condition in the NLP domain, we know that it is not a sufficient condition to be good tokenizer \cite{token_choice}. 

Analogous to the super-pixel resolution in masking-based explanations, tokenization determines the granularity of the representation and thereby the entropy of the queries for TokenSHAP \cite{horovicz2024tokenshap} algorithm. Under a fixed query budget, resolution can be interpreted as the segmentation granularity of the text, where increasing the number of tokens for texts, i.e oversegmentation, enlarges the coalition space that TokenSHAP must explore. From an explainability perspective, oversegmentation may create an impossible region similar to that shown in Fig.~\ref{fig:capacity_breakdown} for TokenSHAP, since the number of token subsets grows beyond what can be effectively explored under a fixed query budget. Since TokenSHAP aims to identify informative tokens for a given task through query-based attribution, excessively fine-grained tokenization may limit the recoverability of faithful token-level explanations. Consequently, increasing resolution leads to an increase in representational entropy that may eventually exceed the amount of budget that can be reliably processed by the model. This perspective suggests that tokenizer design choices, such as vocabulary size and segmentation granularity, may also be studied through the lens of explanation recoverability and query efficiency. A more rigorous empirical investigation of this connection is left for future work.

\section{Conclusion}
\label{sec:conclusion}

This paper develops a reliability theory for masking-based explanations by interpreting the interaction between an explainer and a black-box model as communication over a query channel. In this formulation, the explanation to be recovered plays the role of a message, the masks are channel inputs, and the oracle responses are channel outputs. The complexity of the explanation is quantified by the entropy of the underlying hypothesis set, while the query interface supplies information at most at a capacity per query. A strong converse adapted from channel coding shows that whenever the explanation rate exceeds this identification capacity, exact recovery of the explanation becomes impossible in the limit, regardless of how queries are chosen or how decoding is implemented. An accompanying achievability result demonstrates that a sparse maximum-likelihood decoder can attain vanishing error when the explanation rate is below capacity, establishing a sharp boundary between feasible and infeasible regimes for masking-based XAI.

To connect this theory to practice, we estimate the mutual information of the query channel via Monte Carlo methods and convert it into a finite-query benchmark for reliable explanation. In a controlled setting where exhaustive decoding is tractable, the block error probability of the optimal decoder exhibits a sharp transition near this benchmark, consistent with the theoretical threshold. When the same query data are decoded using Lasso- and OLS-based procedures that emulate LIME and KernelSHAP, the transition shifts to larger query budgets. This shift measures an algorithmic gap: a region where information theory admits reliable explanations but standard convex surrogates still fail to recover the correct support. Dense OLS decoding further requires queries that scale with the ambient dimension rather than the sparsity level, in line with the interpretation of dense regression as ignoring structure in the explanation.

The query-channel viewpoint also clarifies the role of resolution and non-idealities. For image-based explanations, the number of super-pixels determines the entropy of the explanation source and therefore the explanation rate for a fixed query budget. The strong converse then implies the existence of a critical resolution beyond which no decoder can reliably recover the explanation. Empirical super-pixel experiments match this prediction: below the critical segmentation level, sparse decoders produce stable, semantically aligned heatmaps, whereas above it the correlation with the ground truth collapses, even though the underlying model remains well behaved. Additive noise and nonlinear curvature appear as channel degradations that reduce mutual information, shift the waterfall transition, and induce error floors when the information delivered by the query channel falls below the entropy of the explanation.

Overall, the analysis yields operational design rules for masking-based XAI. For a given model and query mechanism, one can trade off query budget, sparsity level, and resolution to remain in a regime where reliable explanation is information-theoretically possible, and interpret instability or error floors as indicators of operating beyond capacity rather than as symptoms of poor optimization. Future work includes adaptive query policies that seek to maximize information per query under computational constraints, extensions to structured and group-sparse explanations, and applications of the query-channel formalism to other perturbation-based explanation methods and alternative query interfaces.

\appendices
\section{Proof of Theorem~\ref{thm:strongconverse}}
\label{app:strong_converse_proof}

This appendix provides a formal proof of Theorem~\ref{thm:strongconverse}. The
argument proceeds in three steps: (i) we restate the query-channel model as an
identification problem over a finite hypothesis set, (ii) we derive a
single-letter upper bound on the information that any sequence of query
strategies can extract, and (iii) we apply a classical strong converse for
discrete memoryless channels to conclude that error probability converges to
one whenever the rate exceeds the parameter-identification capacity
$C^{(S)}$.

\subsection{Model and Code Definition}
\label{app:subsec:model_code}

Let $\mathcal{M}$ be a finite hypothesis set with cardinality
\[
    |\mathcal{M}| = 2^{TR},
\]
and let $M$ be a random variable uniformly distributed over $\mathcal{M}$. Each
hypothesis $m \in \mathcal{M}$ is mapped to a latent parameter vector
$\Phi(m) \in \mathbb{R}^d$. In the sparse setting, $\Phi(m)$ is $k$-sparse and
encodes, among other things, a support set $S(m) \subseteq [d]$.

The query process runs for $T$ time steps. At time $t$, the explainer chooses a
mask $Z_t \in \{0,1\}^d$ according to a (possibly adaptive) query strategy,
which may depend on $M$ and all past observations. Formally, a query strategy
is a sequence of conditional distributions
\[
    \pi_t(z_t \mid m, z^{t-1}, y^{t-1}),
    \qquad t=1,\dots,T,
\]
where $z^{t-1}=(z_1,\dots,z_{t-1})$ and $y^{t-1}=(y_1,\dots,y_{t-1})$. The
observation at time $t$ is generated according to the query-channel law
\[
    Y_t \sim P_{Y|Z,\Phi}(\cdot \mid Z_t, \Phi(M)),
\]
and the joint distribution of $(M,Z^T,Y^T)$ is determined by the prior on $M$,
the strategy $\{\pi_t\}$, and the channel law $P_{Y|Z,\Phi}$.

At the end of $T$ queries, a decoder
\[
    \psi_T : \mathcal{Z}^T \times \mathcal{Y}^T \to \mathcal{M}
\]
produces an estimate $\widehat{M} = \psi_T(Z^T,Y^T)$ of the true hypothesis.
The error probability of this $(2^{TR},T)$ code is
\[
    P_e^{(T)} = \mathbb{P}(\widehat{M} \neq M).
\]
The explanation rate is
\[
    R = \frac{1}{T} \log_2 |\mathcal{M}|
      = \frac{H(M)}{T}
      \quad \text{bits/query},
\]
as in Definition~\ref{def:rate}. The parameter-identification capacity
$C^{(S)}$ is defined in Definition~\ref{def:capacity} as
\[
    C^{(S)} \triangleq \sup_{\pi_Z} I(\Phi;Y \mid Z),
\]
where $\pi_Z$ ranges over all mask distributions with full support on
$\{0,1\}^d$.

Theorem~\ref{thm:strongconverse} asserts that if $R>C^{(S)}$, then there exists
a constant $A>0$, independent of $T$, such that
\[
    P_e^{(T)} \ge 1 - e^{-AT},
\]
for all $(2^{TR},T)$ codes (i.e., for all sequences of strategies and
decoders), and hence $\lim_{T\to\infty} P_e^{(T)} = 1$.

\subsection{Information Bound for Arbitrary Query Strategies}
\label{app:subsec:info_bound}

The first step is to upper bound the mutual information between the hypothesis
$M$ and the observations $(Z^T,Y^T)$ in terms of $C^{(S)}$.

\begin{lemma}
\label{lem:MI_bound}
For any sequence of query strategies $\{\pi_t\}_{t=1}^T$ (possibly adaptive)
and any blocklength $T$,
\begin{equation}
    I(M;Y^T,Z^T) \le T\, C^{(S)}.
    \label{eq:MI_upper_bound}
\end{equation}
\end{lemma}

\begin{IEEEproof}
We start from the chain rule for mutual information:
\begin{align}
    I(M;Y^T,Z^T)
    &= \sum_{t=1}^T I(M;Y_t,Z_t \mid Y^{t-1},Z^{t-1}) \nonumber \\
    &= \sum_{t=1}^T \Big[
        I(M;Z_t \mid Y^{t-1},Z^{t-1})
        + I(M;Y_t \mid Y^{t-1},Z^t)
       \Big].
    \label{eq:chain_MI}
\end{align}

The first term in the brackets is nonnegative, so
\begin{equation}
    I(M;Y^T,Z^T)
    \le \sum_{t=1}^T I(M;Y_t \mid Y^{t-1},Z^t).
    \label{eq:drop_nonnegative}
\end{equation}
We now bound each summand.

Fix $t$ and condition on $(Y^{t-1},Z^t)=(y^{t-1},z^t)$. Given $M$ and $Z_t$,
the observation $Y_t$ is generated according to the single-use query-channel
law:
\[
    Y_t \sim P_{Y|Z,\Phi}(\cdot \mid Z_t, \Phi(M)).
\]
By the memorylessness of the channel, $Y_t$ is conditionally independent of
$(Y^{t-1},Z^{t-1})$ given $(M,Z_t)$, so the following Markov chain holds:
\[
    (M,Y^{t-1},Z^{t-1}) \;\longrightarrow\; (M,Z_t) \;\longrightarrow\; Y_t.
\]
Using data processing and the chain rule, we have
\begin{align}
    I(M;Y_t \mid Y^{t-1},Z^t)
    &\le I(M,Z_t;Y_t \mid Z_t) \nonumber \\
    &= I(M;Y_t \mid Z_t).
    \label{eq:dpi_step}
\end{align}

Next, the message $M$ determines the latent parameter $\Phi = \Phi(M)$, so
$M \to \Phi \to (Z_t,Y_t)$ forms a Markov chain given $Z_t$. In particular,
\[
    M \longrightarrow \Phi \longrightarrow Y_t
    \quad\text{given } Z_t,
\]
and by data processing,
\begin{equation}
    I(M;Y_t \mid Z_t)
    \le I(\Phi;Y_t \mid Z_t).
    \label{eq:Phi_dpi}
\end{equation}

Combining \eqref{eq:dpi_step} and \eqref{eq:Phi_dpi} yields
\begin{equation}
    I(M;Y_t \mid Y^{t-1},Z^t)
    \le I(\Phi;Y_t \mid Z_t).
    \label{eq:per_use_bound}
\end{equation}

By definition of $C^{(S)}$, for any distribution of $Z_t$,
\[
    I(\Phi;Y_t \mid Z_t) \le C^{(S)}.
\]
The query strategy determines the conditional distribution of $Z_t$ given
$(M,Y^{t-1},Z^{t-1})$, but the bound above holds for any such distribution,
since $C^{(S)}$ is a supremum over all input distributions with full support.

Substituting \eqref{eq:per_use_bound} into \eqref{eq:drop_nonnegative}, we
obtain
\[
    I(M;Y^T,Z^T)
    \le \sum_{t=1}^T C^{(S)}
    = T\, C^{(S)},
\]
which proves the lemma.
\end{IEEEproof}

Lemma~\ref{lem:MI_bound} shows that, regardless of how aggressively the explainer
adapts its queries, the total amount of information that can be conveyed about
$M$ through $T$ uses of the query channel is at most $T\,C^{(S)}$ bits.

\subsection{Operational Capacity and Strong Converse}
\label{app:subsec:operational_capacity}

We now relate the single-letter bound to an operational capacity and invoke a
strong converse result from classical channel coding theory.

\begin{definition}[Operational Identification Capacity]
\label{def:Cid}
For the query-channel model described above, define the \emph{operational
identification capacity} as
\begin{equation}
    C_{\mathrm{id}}
    \triangleq
    \sup_{\{\pi_t,\psi_T\}_{T\ge 1}}
    \left\{
        \limsup_{T\to\infty} \frac{1}{T}\log_2|\mathcal{M}_T|
        :
        \limsup_{T\to\infty} P_e^{(T)} < 1
    \right\},
    \label{eq:Cid_def}
\end{equation}
where $\mathcal{M}_T$ is the message set of the $(|\mathcal{M}_T|,T)$ code and
$P_e^{(T)}$ is its error probability. The supremum is taken over all sequences
of query strategies and decoders.
\end{definition}

By standard arguments (see, e.g., Wolfowitz~\cite{wolfowitz1957strong} or
Csisz\'ar and K\"orner), $C_{\mathrm{id}}$ can be expressed in terms of an
information quantity associated with the channel. In our setting, the key step
is to connect $C_{\mathrm{id}}$ and $C^{(S)}$.

\begin{lemma}
\label{lem:Cid_upper_bound}
The operational identification capacity satisfies
\begin{equation}
    C_{\mathrm{id}} \le C^{(S)}.
    \label{eq:Cid_leq_Cs}
\end{equation}
\end{lemma}

\begin{IEEEproof}
Consider any fixed sequence of codes indexed by $T$, with message sets
$\mathcal{M}_T$, priors uniform over $\mathcal{M}_T$, strategies
$\{\pi_t\}$, and decoders $\psi_T$. For blocklength $T$, the rate is
\[
    R_T = \frac{1}{T}\log_2 |\mathcal{M}_T|
        = \frac{H(M_T)}{T},
\]
where $M_T$ is the uniform message.

By Lemma~\ref{lem:MI_bound}, we have
\begin{equation}
    I(M_T;Y^T,Z^T) \le T\, C^{(S)}.
    \label{eq:I_bound_code}
\end{equation}
On the other hand, for any decoder $\psi_T$, the joint distribution of
$(M_T,\widehat{M}_T)$ satisfies the usual relations between error probability
and mutual information (via, e.g., the method of types and divergence
bounds). In particular, strong converse results for discrete memoryless
channels (such as Theorem 5.8.3 in \cite{csiszar2011information} or Wolfowitz’s
original theorem~\cite{wolfowitz1957strong}) imply that if
\[
    R_T > \frac{1}{T} I(M_T;Y^T,Z^T) + \delta,
\]
for some fixed $\delta>0$ and all sufficiently large $T$, then the error
probability $P_e^{(T)}$ must converge to one exponentially fast.

Combining with \eqref{eq:I_bound_code}, we obtain that for any $\epsilon>0$, if
\[
    R_T > C^{(S)} + \epsilon
\]
for all sufficiently large $T$, then $P_e^{(T)} \to 1$ and in fact
$P_e^{(T)} \ge 1 - e^{-A T}$ for some $A>0$. Therefore, any rate strictly
greater than $C^{(S)}$ is not achievable in the sense of
Definition~\ref{def:Cid}, which implies
$C_{\mathrm{id}} \le C^{(S)}$.
\end{IEEEproof}

Lemma~\ref{lem:Cid_upper_bound} shows that the single-letter quantity
$C^{(S)}$ provides an upper bound on the operational capacity. The converse
direction (that rates below $C^{(S)}$ are achievable) is addressed by the
achievability results in Section~\ref{sec:achievability} and is not needed for
the strong converse.

\subsection{Proof of Theorem~\ref{thm:strongconverse}}
\label{app:subsec:proof_main}

We now combine the lemmas above with a classical strong converse statement.

Fix a rate $R$ and consider any sequence of $(2^{TR},T)$ codes with query
strategies and decoders $(\{\pi_t^{(T)}\},\psi_T)$ achieving rate $R$ for each
blocklength $T$. Suppose that $R > C^{(S)}$.

By Lemma~\ref{lem:Cid_upper_bound}, the operational identification capacity
satisfies $C_{\mathrm{id}} \le C^{(S)}$. Hence $R > C_{\mathrm{id}}$. By the
strong converse for discrete memoryless channels (with or without feedback),
there exists a constant $A>0$ (depending only on the channel law) such that the
average error probability of any sequence of codes with rate $R>C_{\mathrm{id}}$
satisfies
\[
    P_e^{(T)} \ge 1 - e^{-A T}
\]
for all sufficiently large $T$. Since $R>C^{(S)}$ implies $R>C_{\mathrm{id}}$,
the same bound holds for any sequence of query strategies and decoders
operating at rate $R$.

\bibliographystyle{IEEEtran}
\bibliography{ref}

\end{document}